\documentclass[10pt,conference]{IEEEtran}

\usepackage[utf8]{inputenc}
\usepackage[T1]{fontenc}
\usepackage{cite}
\usepackage{amsmath,amssymb,amsthm}
\usepackage{graphicx}
\usepackage{booktabs}
\usepackage{xcolor}
\definecolor{darkgreen}{HTML}{008000}
\usepackage{enumitem}
\usepackage{microtype}
\usepackage{placeins}
\usepackage[hidelinks]{hyperref}

\graphicspath{{figures/}}

\newtheorem{proposition}{Proposition}

\theoremstyle{definition}
\newtheorem{definition}{Definition}

\newtheorem{remark}{Remark}

\IEEEoverridecommandlockouts
\title{The Blind Spot Paradox: When Adaptive Classifiers Defeat Drift Detectors}
\author{\IEEEauthorblockN{Raphaël Minato, Fabrice Popineau, Arpad Rimmel, and Bich-Liên Doan\thanks{This research benefited from the support of the Chair ``Artificial intelligence applied to the banking domain and to global health'' led by CentraleSupélec and sponsored by LUSIS.}}
\IEEEauthorblockA{Université Paris-Saclay, CNRS, Laboratoire Interdisciplinaire des Sciences du Numérique, 91190, Gif-sur-Yvette, France\\
Email: \{raphael.minato, fabrice.popineau, arpad.rimmel, bich-lien.doan\}@centralesupelec.fr}}

\begin{document}
\maketitle

\begin{abstract}
  Monitoring concept drift from an adaptive classifier's error stream creates an operational conflict with the model's own update loop. When internal adaptation outpaces evidence accumulation, accuracy recovers before cumulative detectors (CUSUM, Page-Hinkley) can reach threshold. Instrumenting an Adaptive Random Forest (ARF) shows that surviving trees absorb $98.6\%$ of the post-drift error transient through incremental leaf updates alone. The first background tree swap accounts for just $0.71\%$ of this erased error volume, but drops external detection rates by $31$ percentage points. We derive the finite-horizon boundary where cumulative evidence fails to cross threshold and measure a critical magnitude floor ($\Delta e_c = 0.120$) below which false-alarm budgets preclude detection. This failure manifests as missed shifts on stationary streams and false-alarm flooding triggered by internal tree swaps on noisy baselines. We validate on synthetic shifts, ARMA--GARCH series (ProteuS), and tabular benchmarks (BAF, INSECTS); on the synthetic sweep at a standard threshold, the blind spot appears at $\Delta e \approx 0.25$, showing why classical benchmarks like SEA ($\Delta e \le 0.21$) failed to reach it.
\end{abstract}


\section{Introduction}\label{sec:intro}

Coupling an adaptive classifier with an external drift monitor~\cite{gomes_arf_2017, bifet_adwin_2007} creates an unresolved operational conflict: while the classifier adjusts online to restore accuracy, the external monitor accumulates classification errors to alert human operators~\cite{gama_survey_2014, lu_concept_drift_2018}.

In an Adaptive Random Forest (ARF)~\cite{gomes_arf_2017}, internal ADWIN detectors~\cite{bifet_adwin_2007} replace obsolete trees with background trees upon detecting distributional shift. External cumulative monitors (e.g., CUSUM, Page-Hinkley) require sustained error bursts over time to cross their decision threshold. If classifier adaptation operates on a shorter timescale than evidence accumulation, transient error decays before the monitor alarms. We call this the \textbf{blind spot paradox}: rapid adaptation erases the statistical footprint of the change, causing detection rates to drop as shift magnitude increases beyond a certain point. Later threshold crossings, driven by post-recovery random walk, yield uninformative false alarms on already stabilized streams.\footnote{Code, data and the full analysis: \url{https://github.com/TheBlindSpotParadox/TheBlindSpotParadox-Experiments}}

The drift detection literature traditionally treats detection and adaptation as a cooperative sequence.
Surveys spanning two decades, including Klinkenberg \& Joachims~\cite{klinkenberg_2000}, Gama et al.~\cite{gama_survey_2014}, Lu et al.~\cite{lu_concept_drift_2018}, and Bayram et al.~\cite{bayram_2022}, focus primarily on the detector providing trigger signals to initiate model updates.

\noindent\textbf{Preliminary finding.} River~\cite{montiel_river_2021} batches ADWIN's test every $c$ steps, which delays its alarm by a uniform phase of up to $c - 1$ steps; default pipelines run the internal detector at $c = 1$ and the external monitor at $c = 32$. This clock mismatch is an implementation artefact (Section~\ref{sec:clock}), not the mechanism studied here: a PHT, which checks at every step, loses the same race (Section~\ref{sec:complexity}).

\smallskip
\noindent\textbf{Contributions.}
\begin{enumerate}[nosep, label=\textbf{(C\arabic*)}, start=1]
  \item \textbf{Mechanisms of signal erasure.} Modeling the classifier-monitor pairing as a race between stopping times, counterfactual ablations show that incremental leaf updates across surviving trees absorb $98.6\%$ of the error transient. The initial background tree swap erases only $0.71\%$ of the excess error volume, yet triggers a $31$-point drop in external detection.
  \item \textbf{Analytical feasibility and calibration floor.} We derive the finite-horizon boundary crossing conditions for cumulative monitors, establishing an empirical critical floor ($\Delta e_c = 0.120$) below which false-alarm constraints preclude detection, and calibrating static thresholds over the admissible $[0.20, 0.40]$ band.
  \item \textbf{Cross-benchmark failure modes.} Across synthetic shifts ($\Delta e \in [0.02, 0.50]$), GARCH volatility series (ProteuS~\cite{suarez_cetrulo_proteus_2025}), and tabular streams (BAF, INSECTS), we map the divergence between missed detections on stationary data and false-alarm flooding on noisy baselines, showing why historical benchmarks like SEA ($\Delta e \le 0.21$) failed to reach the blind spot boundary at the crossover threshold.
\end{enumerate}

\section{Background}\label{sec:background}

Drift monitoring splits into two paradigms~\cite{gama_survey_2014, lu_concept_drift_2018}: monitoring the input distribution $P(X)$, or monitoring the posterior $P(Y|X)$ through a classifier's binary error stream $e_t \in \{0,1\}$, which assumes that each true label arrives before the next prediction. This paper concerns the second, where an adaptive classifier produces the detector's input. Throughout, we call the external component the \emph{monitor} and reserve \emph{detector} for the classifier's internal ADWIN and for algorithm families in the literature sense.

\subsection{Detectors and the Adaptive Random Forest}\label{sec:adwin}
\textbf{ADWIN}~\cite{bifet_adwin_2007} maintains a sliding window split into two sub-windows and declares drift when their means diverge beyond a Hoeffding bound. River~\cite{montiel_river_2021} evaluates this test only every $c$ steps ($c=32$ by default); its Page-Hinkley Test (PHT) evaluates at every step.

\noindent\textbf{The PHT}~\cite{page_1954} is a CUSUM detector~\cite{basseville_1993}: an alarm fires when the running maximum of its accumulated statistic exceeds a threshold $\lambda > 0$ (governing the false-alarm rate~\cite{tartakovsky_2014}), net of a minimum-magnitude tolerance $\delta_P > 0$. River's implementation runs two-sided by default, alarming on a rise or a fall of the mean; we analyze the one-sided variant that accumulates excess error above a fixed reference and resets at zero (Eq.~\eqref{eq:cusum}). By Wald's identity, after an abrupt rise $\Delta e > \delta_P$ of the error rate (the \emph{effective error jump}) the expected time to alarm is $\mathbb{E}[\tau_{\mathrm{det}}] \approx \lambda / (\Delta e - \delta_P) = \mathcal{O}(1/\Delta e)$.

\noindent\textbf{Hoeffding Trees (HT)}~\cite{gama_survey_2014} are incremental decision trees, \emph{non-adaptive} once converged. \textbf{Adaptive Random Forests (ARF)}~\cite{gomes_arf_2017} maintain $M$ HTs, each paired with an internal ADWIN: upon an internal drift warning the ARF pre-trains a background tree and, upon confirmation, instantly replaces the active tree.

\subsection{Related Work}\label{sec:related}
Sequential change-point methods under independent observations are well understood: CUSUM achieves minimax optimality~\cite{moustakides_1986} and Lorden's criterion~\cite{lorden_1971} characterizes worst-case delay.
Detector benchmarks~\cite{goncalves_2014, lukats_2025} run under homoscedastic, i.i.d.\ conditions, on standard generators (SEA, Hyperplane, RandomRBF) or pre-generated bitstreams (e.g., SINE1, MIXED), with no classifier in the loop; work introducing adaptive ensembles such as ARF~\cite{gomes_arf_2017} reports predictive accuracy only.
The closest prior observation to our race condition is that of Stirling et~al.~\cite{stirling_2018}, who studied internal-detector choices in Hoeffding Adaptive Trees but evaluated classification accuracy only. Read~\cite{read_2018} flagged destructive adaptation in Hoeffding-tree ensembles as a \emph{classification} inefficiency; seen from the monitoring side, the same mechanism erases the monitor's evidence. Starvation (Section~\ref{sec:starvation}) concerns every cumulative-evidence monitor whose stopping time scales as $\Omega(\lambda / \Delta e)$---DDM~\cite{gama_ddm_2004}, EDDM~\cite{baena_eddm_2006}, ECDD~\cite{ross_ecdd_2012}. We confirm it empirically for PHT (Sections~\ref{sec:blindspot}--\ref{sec:experiments}) and EDDM (Section~\ref{sec:proteus}); for the other members it follows from the same accumulation argument (Prop.~\ref{prop:starvation}). Windowed and distributional monitors (ADWIN~\cite{bifet_adwin_2007}, KSWIN~\cite{raab_kswin_2020}) do not accumulate evidence, so this mechanism does not apply to them. Escaping this mechanism does not necessarily make them more reliable (Remark~\ref{rem:agnostic}).

\section{The Blind Spot Paradox}\label{sec:blindspot}

We instrument the ARF of River 0.23.0~\cite{montiel_river_2021} and compare the exact time step of each internal tree swap against a concurrent external monitor reading the \emph{same} error stream.

\subsection{Preliminary: ADWIN's Evaluation-Clock Delay}\label{sec:clock}

Because ADWIN evaluates its test only at multiples of the clock parameter $c$, we model the measured detection delay for an abrupt drift at unknown time $\tau^*$ as:
\begin{equation}\label{eq:clock}
  \mathrm{ADD} = \bigl(c - 1 - (\tau^* \bmod c)\bigr) + \tau_{\mathrm{stat}}\,,
\end{equation}
where $\tau_{\mathrm{stat}}$ is the irreducible statistical delay; the phase term counts completed evaluation intervals, a counting convention under which the decomposition is exact. For large drifts, $\tau_{\mathrm{stat}}$ is a small constant---ADWIN's Hoeffding bound certifies a shift of magnitude $\Delta e$ with $\mathcal{O}(1/\Delta e^{2})$ post-change samples~\cite{bifet_adwin_2007}---and delay is dominated by phase alignment: for uniform $\tau^*$, $\mathrm{ADD} \sim \mathcal{U}\{\tau_{\mathrm{stat}},\, \tau_{\mathrm{stat}} + c - 1\}$ (mean ${\approx}23$ at $c{=}32$). This is a deterministic artefact of \emph{mismatched} clocks, not an ADWIN property.

\subsection{The Race Condition}\label{sec:race}

Let $N_{\mathrm{swap}}(t)$ denote the cumulative number of internal tree replacements (tree swaps) in the adaptive classifier up to time $t$. This quantity is directly observable via River's internal \texttt{\_drift\_tracker} counter.

\begin{definition}[Adaptation Onset, Preemption \& Blind Spot]\label{def:tau_arf}
  The \textbf{adaptation onset} $\tau_{\mathrm{ARF}}$ is the delay from drift onset $\tau^*$ to the first internal tree swap: $\tau_{\mathrm{ARF}} := \inf\bigl\{t > \tau^* : N_{\mathrm{swap}}(t) > N_{\mathrm{swap}}(\tau^*) \bigr\} - \tau^*$.
  \textbf{Preemption} occurs when the first swap---the forest's first committed structural response---precedes the external alarm at $\tau_{\mathrm{det}}$: $\tau_{\mathrm{ARF}} < \tau_{\mathrm{det}}$. This event orders the two detections; it is observable online but does not measure the erasure, which the ablations below attribute mostly to incremental learning.
  A \textbf{blind spot} occurs when the erasure itself completes before the alarm: writing $W$ for the length of the post-drift error transient (formalized in Section~\ref{sec:starvation}), $W < \tau_{\mathrm{det}}$---the monitor is starved of the drift's evidence.
\end{definition}

\noindent\textbf{Drift generator.}
At time $\tau^*$ the concept boundary shifts from $x_0 + x_1 > 0$ to $x_0 + x_1 > b$ ($b > 0$); for $x_0, x_1 \sim \mathcal{N}(0,1)$ the theoretical instantaneous error jump is $\Delta e = \Phi(b/\sqrt{2}) - 0.5$. Estimating $\Delta e$ empirically by averaging classification errors over a post-drift window fails: because the ARF adapts rapidly within that window, the post-change error burst is diluted into the recovered stream (at $b = 4.0$, a standard sliding-window average yields $\widehat{\Delta e} \approx 0.02$ against the true initial shock $\Delta e \approx 0.50$). All reported experiments therefore evaluate pipelines against the theoretical pre-adaptation jump $\Delta e$.

\noindent\textbf{Decomposing internal adaptation mechanisms.}
To isolate the influence of tree updates from tree swaps, four counterfactual ablations replay each stream from the change instant $\tau^*$ under identical historical states: a frozen ensemble (no structural or parametric updates), incremental learning without replacements, learning restricted to the initial replacement, and the default unablated ARF. Measuring the differential reduction in excess error shows that incremental learning in the surviving trees accounts for $98.6\%$ $[98.5, 98.8]$ of the signal erasure. In contrast, the first tree substitution represents merely $0.71\%$ $[0.63, 0.78]$ of erased error volume, followed by $0.56\%$ for subsequent swaps (medians of per-cell shares, for this order of enabling the mechanisms). This first swap nevertheless weighs decisively on the monitoring outcome: under $\lambda = 50$ and $\Delta e = 0.33$, an external CUSUM monitor achieves $0/100$ detections with the default ARF, recovering to $18/100$ when only the first swap is permitted, and to $49/100$ when all swaps are blocked. The delay $\tau_{\mathrm{ARF}}$ thus marks adaptation onset rather than completion. The ablations also show that the first swap acts as an early marker rather than an indicator of full recovery: the Spearman rank correlation between the first swap and the time at which excess error halves reaches $0.73$ in the mid-magnitude band ($0.15 < \Delta e \le 0.35$) and $0.78$ in the strong band ($\Delta e > 0.35$), but drops to $0.13$ (mid band) with the time of full stabilization back to pre-drift accuracy ($-0.14$ for $\Delta e \le 0.15$).

\subsection{Ensemble Acceleration}\label{sec:hydra}

The ARF maintains $M$ trees, each with its own internal ADWIN.
Upon drift, tree $i$ detects and swaps at time $\tau_i$, and the ensemble starts adapting at the first of these swaps:
\begin{equation}\label{eq:hydra}
  \tau_{\mathrm{ARF}} = \min_{1 \le i \le M} \tau_i\,.
\end{equation}
If the delays $\tau_i$ were independent with common distribution $F$, order statistics~\cite{david_order_2003} would define the expected ensemble delay as $\mathbb{E}[\tau_{(1:M)}] = \int_0^{\infty} (1 - F(t))^M \, dt$.
Under memoryless delays (exponential $F$), this yields an exact $M$-fold speedup: $\mathbb{E}[\tau_{(1:M)}] = \mathbb{E}[\tau_1] / M$.
More generally, for any continuous distribution whose density $f$ satisfies $f(0) > 0$, the first-order approximation near zero provides the asymptotic scaling $\mathbb{E}[\tau_{(1:M)}] \sim 1 / (M f(0))$ as $M \to \infty$.
However, because trees share one stream and one Poisson generator, they are not statistically independent (Section~\ref{sec:starvation_boundary}), requiring direct empirical quantification of the ensemble speedup.

\textbf{Empirical validation.}
We evaluate the adaptation delay of a single Hoeffding Adaptive Tree (HAT, $M = 1$)~\cite{bifet_hat_2009} against the ARF ensemble ($M = 10$).
Because some HAT runs do not trigger before the post-drift horizon ($T = 4{,}000$), their delays are right-censored. We therefore evaluate speedup using restricted mean survival times (RMST, $\mathbb{E}[\min(\tau, T)]$) and median survival times estimated via Kaplan--Meier.
At moderate drift ($\Delta e = 0.14$), the ARF achieves a speedup of at least $4.12\times$ on restricted means ($95\%$ CI $[3.45, 4.96]$) and $5.32\times$ on medians.
At strong drift ($\Delta e = 0.33$), the acceleration reaches at least $7.99\times$ on restricted means ($[6.40, 9.72]$), whereas the median ratio contracts to $3.10\times$.
These mean ratios represent rigorous lower bounds on the true acceleration factor, as HAT's true mean is right-truncated by $T$ while ARF is never censored.
The divergence between mean and median dynamics reveals that the delay distributions do not scale homothetically across drift magnitudes: as $\Delta e$ shifts, heavy-tailed delays emerge in HAT, pulling the mean ratio upward while the median ratio compresses.

Fitting a power law $\tau \approx C\,(\Delta e)^{-\alpha}$ across the full grid ($\Delta e \in [0.02, 0.50]$) yields flattened exponents $\hat\alpha_{\mathrm{HAT}} \approx 1.02$ and $\hat\alpha_{\mathrm{ARF}} \approx 0.98$, capped by background noise swaps at low magnitudes ($\Delta e < 0.10$, Remark~\ref{rem:bgswap}) where false alarms saturate delays.
Excluding this noise band ($\Delta e \ge 0.10$) and fitting median delays robust to censoring, ARF scaling steepens to $\hat\alpha = 1.77$ ($95\%$ bootstrap CI $[1.67, 1.90]$).
Log-space polynomial regression ($\ln \tau_{\mathrm{ARF}} = \beta_0 + \beta_1 \ln \Delta e + \beta_2 (\ln \Delta e)^2$) reveals a positive quadratic coefficient $\beta_2 = 0.50$ $[0.30, 0.76]$, confirming that the response curves on log--log axes rather than following an invariant power law.
The resulting speedup is magnitude-dependent; this parallel stopping dynamic ($\tau_{\mathrm{ARF}} = \min_i \tau_i$) mirrors the multichart sequential scheme of Tartakovsky~\cite{tartakovsky_2005_multichart}, which stops at the first of $M$ parallel charts to cross threshold.

\subsection{Starvation: Transient Signal Erasure}\label{sec:starvation}

When the ARF adapts, internal learning collapses a persistent drift into a transient error signal of length $W$ ($[\tau^*, \tau^* + W]$) at height $\Delta e$---an instance of \emph{transient change detection}~\cite{nikiforov_2012,tartakovsky_2023_transient} where tests for persistent shifts degrade.
A CUSUM detector with threshold $\lambda$ and tolerance $\delta_P$ accumulates evidence $S_t$:
\begin{equation}\label{eq:cusum}
  S_t = \max\bigl(0,\, S_{t-1} + e_t - p_0 - \delta_P\bigr)\,.
\end{equation}
During the transient, each increment contributes net drift $\mathbb{E}[e_t - p_0 - \delta_P] = \Delta e - \delta_P > 0$ (canceling baseline $p_0$).
Afterwards ($t > \tau^* + W$), accuracy recovers ($p_t = p_0$), increments turn negative ($-\delta_P < 0$), and evidence decays toward zero. The monitor is \emph{starved} (Definition~\ref{def:tau_arf}) when the transient ends before $S_t$ reaches $\lambda$.

\begin{proposition}[Starvation: Finite-Horizon Boundary Crossing]\label{prop:starvation}
  Write $W$ for the transient length, $\mu := \Delta e - \delta_P > 0$, $S_{\tau^*} = s_0 \ge 0$, and $S^* := \sup_{t\in[\tau^*,\, \tau^* + W]} S_t$. Then
  \begin{equation}\label{eq:markov_bound}
    \mathbb{P}\bigl(S^* \ge \lambda\bigr) \;\le\; W \exp\!\Bigl(-\tfrac{2}{W}\bigl(\lambda - s_0 - \mu W\bigr)_{+}^{2}\Bigr)\,.
  \end{equation}
  Equivalently, the monitor misses the transient with probability at least $1-\varepsilon$ as soon as
  \begin{equation}\label{eq:starve_boundary}
    \lambda \;\ge\; s_0 + \mu W + \sqrt{\tfrac{W}{2}\,\ln(W/\varepsilon)}\,.
  \end{equation}
  The second term ($\mu W$) is the mean accumulated evidence; the third is the stochastic fluctuation margin, dominating whenever $\mu\sqrt{W} \lesssim 1$.
\end{proposition}

\begin{proof}
  Let $\xi_i := e_i - (p_0 + \Delta e)$ be the centered increments of unit support ($\sup \xi_i - \inf \xi_i = 1$). Unrolling~\eqref{eq:cusum} gives $S_t = \max\bigl(s_0 + \sum_{i \le t}(\xi_i + \mu),\ \max_{k < t}\sum_{k < i \le t}(\xi_i + \mu)\bigr)$, so $\{S^* \ge \lambda\} \subseteq \bigcup_{\tau^* \le k < \tau^* + W} \bigl\{\max_{k < j \le \tau^*+W}\sum_{k < i \le j}\xi_i \ge \lambda - s_0 - \mu W\bigr\}$. Doob's maximal inequality applied to the exponential submartingale with Azuma--Hoeffding bounds for unit-support martingale increments~\cite{hoeffding_1963,tartakovsky_2014} yields~\eqref{eq:markov_bound}; solving for $\lambda$ yields~\eqref{eq:starve_boundary}.
\end{proof}

\begin{remark}[Finite horizon and post-transient dynamics]\label{rem:sufficient}
  Eq.~\eqref{eq:markov_bound} bounds the boundary-crossing probability during the active transient window only; it does not imply permanent immunity against threshold crossings. Once adaptation restores the error rate to baseline ($t > \tau^* + W$), the drift becomes negative ($-\delta_P < 0$). However, the zero-reflected CUSUM statistic $S_t$ remains positive recurrent and will eventually cross any finite threshold $\lambda$. By classical CUSUM renewal theory~\cite{tartakovsky_2014,page_1954}, the expected time to such an event scales exponentially with the threshold ($\mathbb{E}[\tau_{\mathrm{FA}}] = \Theta(e^{c \lambda})$). Any alarm raised after $\tau^* + W$ is entirely uninformative regarding the change-point, representing one of the two operational mechanisms driving false-alarm flooding (Remark~\ref{rem:flooding}). Furthermore, the $\mathcal{O}(\sqrt{W})$ diffusion term in~\eqref{eq:starve_boundary} explains why the empirical crossover in Section~\ref{sec:crossover} manifests as a progressive stochastic transition band (Fig.~\ref{fig:crossover_and_solution}, center) rather than a discontinuous step function, alongside the bimodal $\tau_{\mathrm{ARF}}$ distribution in the Safe Zone and low-magnitude noise swaps (Remark~\ref{rem:bgswap}).
\end{remark}

\begin{remark}[Scope across detector families]\label{rem:agnostic}
  Proposition~\ref{prop:starvation} applies to cumulative-evidence monitors, including the CUSUM/PHT family and statistical process control charts such as DDM~\cite{gama_ddm_2004}, EDDM~\cite{baena_eddm_2006}, and ECDD~\cite{ross_ecdd_2012}. These detectors require persistent deviations and fail when model adaptation truncates the error signal. In contrast, distributional monitors such as KSWIN~\cite{raab_kswin_2020} evaluate a two-sample Kolmogorov--Smirnov test between consecutive sub-windows rather than integrating residuals. However, windowed tests exhibit different limitations: they lose statistical power when transient length $W$ is small relative to window size. Furthermore, on streams with non-zero baseline error ($p_0 > 0$, as in the Bernoulli streams of Section~\ref{sec:proteus}), the continuous approximations required for binary streams yield spurious alerts prior to drift onset.
\end{remark}

\textbf{Empirical validation and late alarms.}
With the instrumented ARF ($M = 10$, $c_{\mathrm{int}} = 1$) and an external StrictCUSUM (the fixed-$p_0$ CUSUM of Eq.~\eqref{eq:cusum}, $\lambda = 50$, $\delta_P = 0.01$), the external monitor detects $0\%$ of the drifts across 100 seeds for $\Delta e > 0.24$ (Figure~\ref{fig:pht_scenarios}(A)): total starvation.
Recording both stopping times $(\tau_{\mathrm{ARF}}, \tau_{\mathrm{det}})$ shows a second failure mode at standard thresholds (e.g., $\lambda = 25$): the CUSUM still alarms in more than $95\%$ of runs, but \emph{after} the first internal swap (preemption, $\tau_{\mathrm{ARF}} < \tau_{\mathrm{det}}$, in $88\%$ of runs). We call this a \textbf{late alarm}: it reports a genuine drift, but only after the ensemble has committed its first structural adaptation ($\tau_{\mathrm{ARF}} < \tau_{\mathrm{det}}$) and partially absorbed the error transient. In Figure~\ref{fig:race_condition} (specifically panel $\lambda = 25.0$), this regime appears as the dense cluster of detections occurring below the diagonal identity line.

\begin{figure*}[t]
  \centering
  \includegraphics[width=0.98\textwidth]{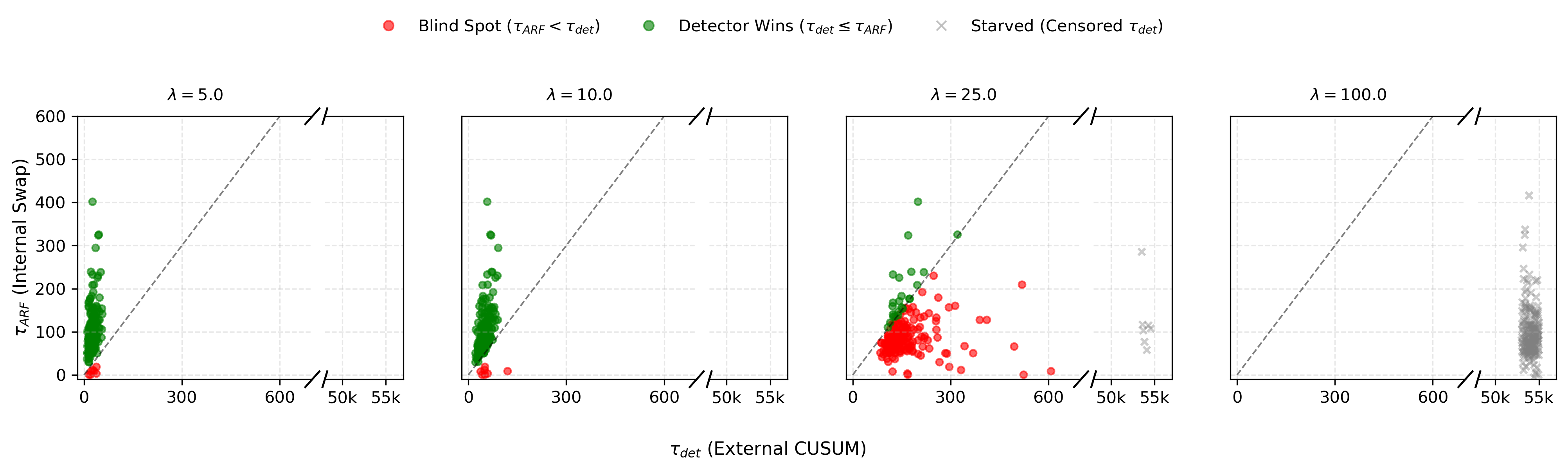}
  \caption{Race between the internal ($\tau_{\mathrm{ARF}}$) and external ($\tau_{\mathrm{det}}$) stopping times at $\Delta e{=}0.25$ ($200$ seeds/panel); starved runs ($\tau_{\mathrm{det}}$ censored at horizon $50{,}000$) sit in the broken right-hand zone. The transition $\lambda^*{\approx}15$--$25$ separates the monitor-wins, preemption (late-alarm) and total-starvation regimes.}
  \label{fig:race_condition}
\end{figure*}

\subsection{The Clock-Mismatch Artefact}\label{sec:hardware}

We isolate the impact of the evaluation clock on the real ARF across three configurations ($100$ seeds $\times$ $20$ magnitudes). Under \emph{mismatched} clocks where the classifier evaluates adaptation 32 times more frequently than the external monitor ($c_{\mathrm{int}}{=}1$, $c_{\mathrm{ext}}{=}32$---River's default), the ARF, accelerated by its parallel ensemble stopping dynamic, preempts the batched external ADWIN: miss rates exceed $50\%$ for $\Delta e > 0.35$ and \emph{increase} with magnitude, the counter-intuitive artefact signature. This is an artefact of clock disparity rather than an intrinsic failure of ADWIN. Under \emph{matched} clocks ($c_{\mathrm{int}}{=}c_{\mathrm{ext}}{=}1$), the external ADWIN reliably detects before adaptation ($\tau_{\mathrm{det}} \le \tau_{\mathrm{ARF}}$) for strong drifts ($\Delta e > 0.30$), keeping missed detections below $5\%$. (The residual $11\%$ miss rate observed at $\Delta e \approx 0.19$ reflects the statistical detection boundary of the weak-signal regime rather than preemption.) This advantage arises because the external monitor observes the aggregated \emph{ensemble} error averaged over $M{=}10$ trees, yielding a higher signal-to-noise ratio than any single tree, whereas a single-tree HAT baseline suffers ${\approx}40\%$ missed detections under the same matched clock. The decoupled configuration ($c_{\mathrm{int}}{=}32$, $c_{\mathrm{ext}}{=}1$) confirms detection before adaptation for $\Delta e > 0.30$, consistent with the \hyperref[def:decoupling]{Decoupling Principle}.

\subsection{A Cumulative Monitor at Three Thresholds}\label{sec:complexity}

Figure~\ref{fig:pht_scenarios} illustrates the race between the instrumented ARF ($M = 10$, $c_{\mathrm{int}} = 1$) and an external StrictCUSUM across three control thresholds spanning representative operational profiles ($\Delta e \in [0.02, 0.50]$).

At $\lambda = 50$ (Fig.~\ref{fig:pht_scenarios}A, conservative profile), the monitor detects in $<$1\% of runs across all magnitudes: accelerated adaptation purges the error transient before the cumulative sum can reach threshold (total starvation). At $\lambda = 25$ (Fig.~\ref{fig:pht_scenarios}B), detection exhibits a U-shaped failure profile: misses are total at low magnitude ($\Delta e < 0.11$, weak signal), fall to $8\%$ near $\Delta e = 0.33$ while most alarms still follow the first swap (bars), and climb to $96\%$ at $\Delta e = 0.50$ as adaptation accelerates, the signature of the blind spot paradox. At $\lambda = 8$ (Fig.~\ref{fig:pht_scenarios}C, hyper-reactive profile), the monitor detects reliably for $\Delta e > 0.15$, but this sensitivity incurs prohibitive false-alarm rates as soon as baseline noise is introduced.

\begin{figure*}[t]
  \centering
  \includegraphics[width=0.31\textwidth]{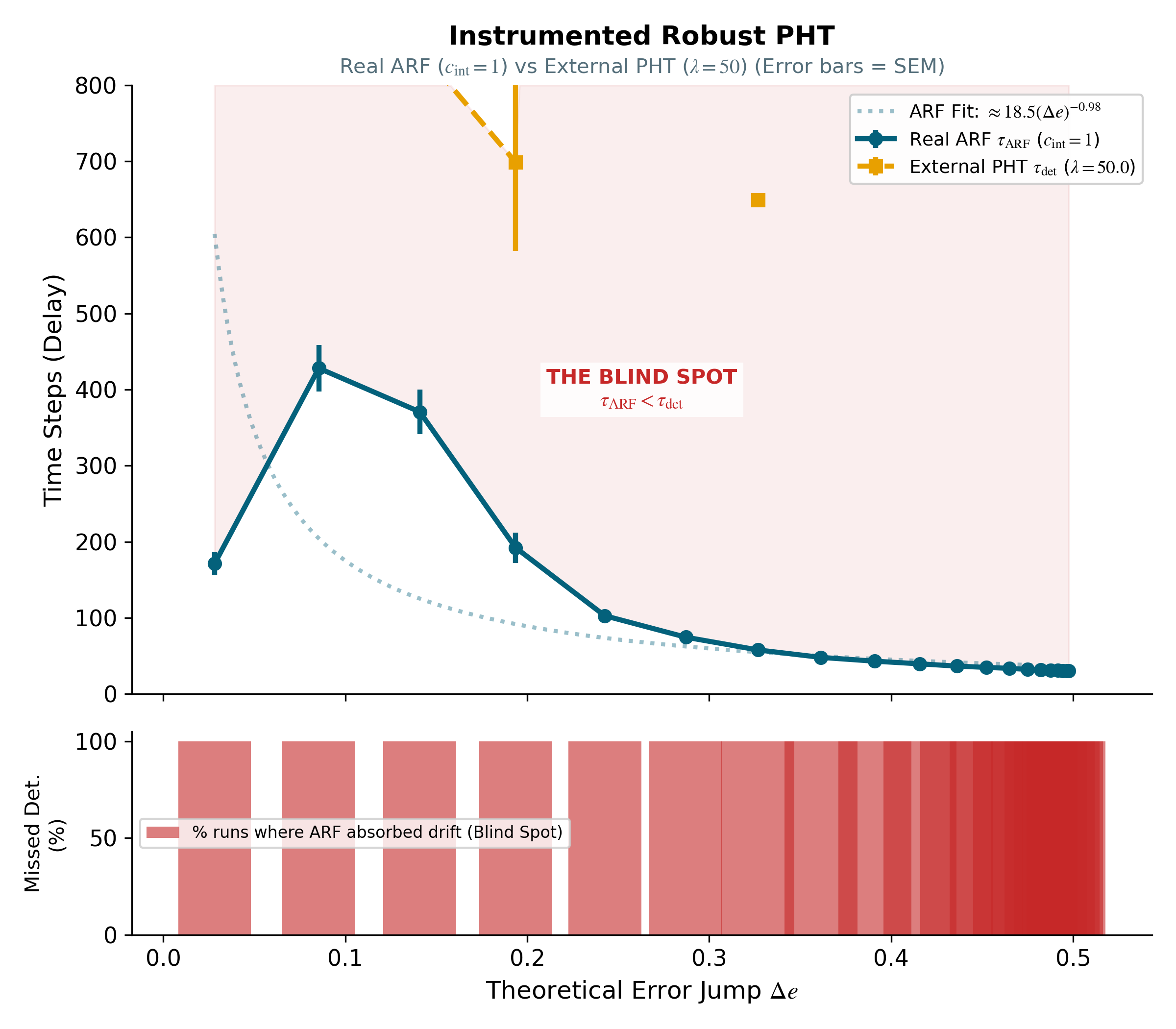}
  \includegraphics[width=0.31\textwidth]{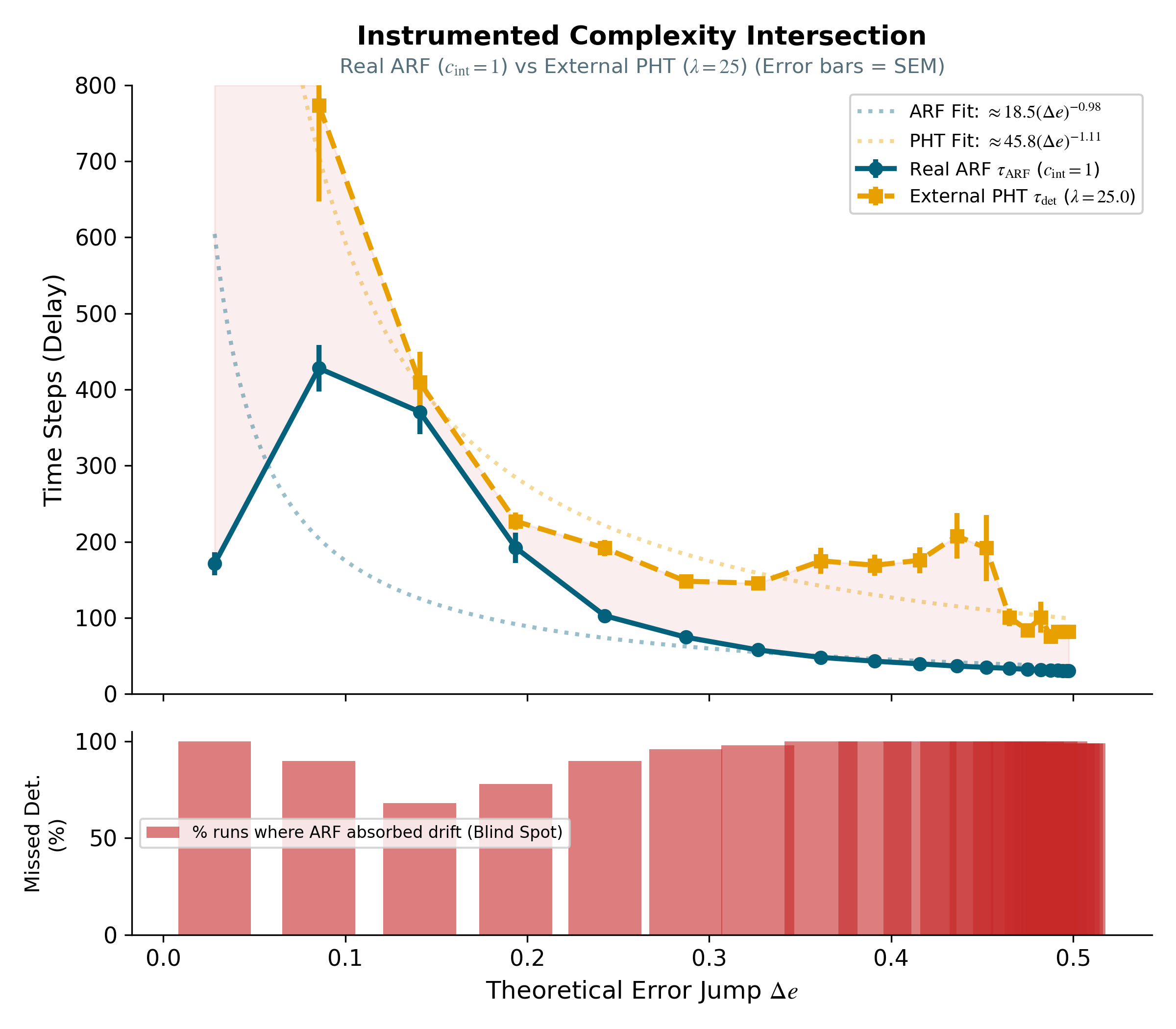}
  \includegraphics[width=0.31\textwidth]{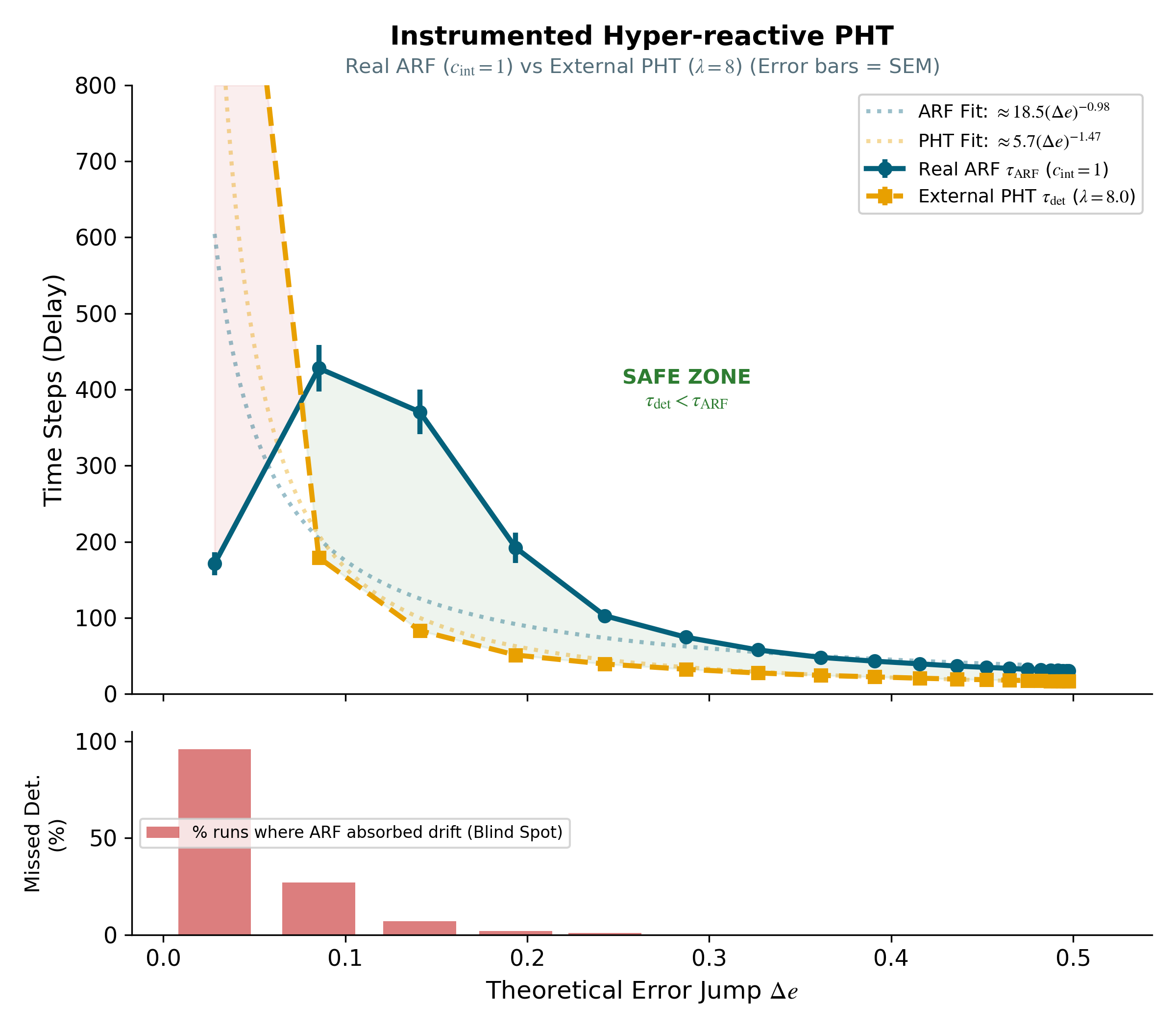}
  \caption{Instrumented race, ARF ($M{=}10$) vs.\ external StrictCUSUM. (A)~$\lambda{=}50$: $100\%$ miss. (B)~$\lambda{=}25$: paradoxical miss increase with $\Delta e$ (bars: runs where ARF absorbed drift before alarm). (C)~$\lambda{=}8$: reliable detection at $\Delta e{>}0.15$.}
  \label{fig:pht_scenarios}
\end{figure*}

\begin{remark}[Background swap rate]\label{rem:bgswap}
  At very low drift magnitudes ($\Delta e < 0.10$), $\tau_{\mathrm{ARF}}$ exhibits non-monotonic behavior: the ARF occasionally swaps trees even under stationary conditions due to the high evaluation frequency of internal ADWIN detectors at $c_{\mathrm{int}} = 1$.
  These ``noise swaps'' saturate delay at the ensemble in-control run length ($\mathrm{ARL}_0 / M \approx 170$), causing the downward hook observed at $\Delta e < 0.10$ in Figure~\ref{fig:pht_scenarios} instead of following the asymptotic power-law fit. Each noise swap introduces an immature background tree, generating a localized error burst that can trigger spurious external drift alarms (Remark~\ref{rem:flooding}).
\end{remark}

\begin{remark}[False-alarm flooding as the dual regime of starvation]\label{rem:flooding}
  Adaptive monitoring fails through two distinct regimes: starvation (loss of recall) and flooding (loss of precision). Flooding stems from two independent mechanisms: post-transient diffusion after a true drift (Remark~\ref{rem:sufficient}), and localized error bursts from premature background swaps on noisy baselines ($p_0 > 0$, Remark~\ref{rem:bgswap}). On clean stationary streams with internal detectors at $c_{\mathrm{int}}=1$, surviving trees absorb error spikes via incremental leaf updates (Section~\ref{sec:race}), cutting off the cumulative monitor's signal ($\mathrm{recall} \to 0$, Proposition~\ref{prop:starvation}). Conversely, under low signal-to-noise ratios or gradual transitions, sample variance triggers unjustified tree replacements. Each premature substitution introduces an under-trained background tree, producing a transient surge in classification errors that the external monitor mistakes for environmental drift, collapsing precision ($\mathrm{precision} \to 0$). On the \emph{gradual\_balanced} benchmark with $p_0 > 0$, pairing a calibrated external PHT with an ARF at $c_{\mathrm{int}}=1$ inflates external alarms to $86$ for a single ground-truth transition, compared to only $7$ for the static HT baseline (Table~\ref{tab:real_data_summary}). Admissible thresholds are therefore bounded on both sides: from below by the false-alarm budget set by stream noise, and from above by the transient evidence ceiling.
\end{remark}

\subsection{A Distribution-Free Race Bound}\label{sec:starvation_boundary}

\noindent\textbf{Tree dependence.}
The $M$ trees in the ARF train on the \emph{same} stream via Poisson bootstrap (Oza bagging with $\lambda = 6$~\cite{gomes_arf_2017}), rendering delays $\tau_1, \ldots, \tau_M$ positively dependent. Conditioning on the stream realisation would restore independence if the per-tree randomness were private, and a Jensen argument on $x \mapsto x^M$ would then bound the minimum. Yet, it is not private: River draws the Poisson weights of every member of the forest from one shared generator, so conditional independence fails at the implementation. Consequently, the race bound must hold under arbitrary dependence across trees:

\begin{proposition}[Distribution-free race bound]\label{prop:starvation_boundary}
  Let $\tau_1, \ldots, \tau_M$ be the internal detection delays of the $M$ trees, $F$ the marginal CDF of one of them, and $\tau_{\mathrm{ARF}} = \min_i \tau_i$. For every horizon $s$ and whatever the dependence between the $\tau_i$,
  \begin{equation}\label{eq:boole_bound}
    F(s) \;\le\; \mathbb{P}\bigl(\tau_{\mathrm{ARF}} \le s\bigr) \;\le\; \min\bigl(1,\, M\,F(s)\bigr)\,.
  \end{equation}
\end{proposition}

\begin{proof}
  The lower bound is monotonicity of measure applied to one member. The upper bound is Boole's inequality on the union $\bigcup_i \{\tau_i \le s\}$.
\end{proof}

At operating points reported here, $M F(s) > 1$ and~\eqref{eq:boole_bound} saturates, bounding miss probability trivially by one. We therefore report empirical detection rates. With the instrumented ARF at $\lambda = 50$, the monitor misses $>99\%$ of drifts at $\Delta e = 0.33$ (Fig.~\ref{fig:pht_scenarios}A). At $\lambda = 25$, the detection rate exhibits an inverted U-shape ($0.00$ below $\Delta e = 0.11$, $0.92$ near $\Delta e = 0.33$, $0.04$ at $\Delta e = 0.50$), matching the U-shaped miss profile (Section~\ref{sec:complexity}) and evidence ceiling (Section~\ref{sec:crossover}). This non-monotonicity contradicts canonical accumulation models where delay scales strictly as $\mathcal{O}(1/\Delta e)$.

\subsection{The Decoupling Principle}\label{sec:decoupling}

\begin{definition}[Decoupling Principle]\label{def:decoupling}

  \mbox{}\par
  \smallskip
  \emph{(i)~Sufficient condition.} A monitoring pipeline satisfies the \textbf{Decoupling Principle} when the external monitor is configured to be \emph{strictly more reactive} than the classifier's internal adaptation, in the unit of its own family (a clock or a threshold):
  \begin{equation}\label{eq:decoupling}
    \begin{cases}
      c_{\mathrm{ext}} < c_{\mathrm{int}}                                                 & \text{(clocked)},    \\
      \lambda < \min_{\Delta e \in [\Delta e_{\min}, \Delta e_{\max}]} S_{\max}(\Delta e) & \text{(cumulative)},
    \end{cases}
  \end{equation}
  where $S_{\max}(\Delta e)$ is the peak of the evidence $S_t$ over the post-change horizon for a drift of magnitude $\Delta e$. The evidence peaks when adaptation completes, so a threshold below the peak is crossed before the evidence is erased; the minimum runs over the range because the peak is not monotone in $\Delta e$ (Section~\ref{sec:crossover}). The clock condition removes the phase delay of Eq.~\eqref{eq:clock}, and in Section~\ref{sec:hardware} it suffices for $\Delta e > 0.30$.

  \smallskip
  \emph{(ii)~Operational calibration.} The cumulative condition in~\eqref{eq:decoupling} is calibrated over $[\Delta e_{\min}, \Delta e_{\max}]$ at reliability $1-\alpha$ via four offline steps:
  \begin{enumerate}[nosep, leftmargin=*, label=\textbf{(\arabic*)}]
    \item \textbf{Target range.} Fix the operational interval $[\Delta e_{\min}, \Delta e_{\max}]$;
    \item \textbf{Peak recording.} For each grid point $\Delta e$, record peak evidence $S_{\max}$ under synthetic abrupt drift;
    \item \textbf{Quantile extraction.} Extract empirical quantile $q_{\alpha}(S_{\max}(\Delta e))$;
    \item \textbf{Admissibility check.} Set $\lambda_{\mathrm{op}} := \min_{\Delta e} q_{\alpha}(S_{\max}(\Delta e))$, admissible if $\lambda_{\mathrm{op}} \ge \lambda_{\mathrm{FA}}$ (the false-alarm budget).
  \end{enumerate}
  \smallskip
  \noindent Calibrated on synthetic Gaussian streams (Section~\ref{sec:race}; StrictCUSUM, $\delta_P = 0.01$, $\alpha = 0.05$), the range is bounded on both sides. Below floor $\Delta e_c = 0.120$ $[0.114, 0.127]$, transient evidence cannot cross any threshold meeting the false-alarm budget, leaving an empty admissible set. On $[0.20, 0.40]$, calibration yields $\lambda_{\mathrm{op}} = 21.93$ $[19.88, 22.40]$. On $[0.20, 0.50]$, accelerated adaptation at $\Delta e = 0.50$ depresses peak evidence below $\lambda_{\mathrm{FA}}$ in $62.6\%$ of $10{,}000$ bootstrap replicates, leaving the calibration unidentified and capping the range at $\Delta e_{\max} = 0.40$. The method requires only offline drift simulation, without an analytical transient model.
\end{definition}

For cumulative-evidence monitors, a \textbf{non-adaptive classifier} (e.g., Hoeffding Tree) has no internal drift mechanism, so \textbf{no race condition is involved}.

\section{Empirical Validation}\label{sec:experiments}

All experiments ran on an AMD EPYC 8224P (24 cores, 48 threads) with 192\,GB RAM in ${\approx}5$\,h total.

\subsection{Beyond I.I.D. Streams: Heteroscedastic ARMA--GARCH Benchmarks}\label{sec:proteus}

\noindent\textbf{Experimental setup.}
All experiments use the ProteuS generator~\cite{suarez_cetrulo_proteus_2025}, simulating GARCH(1,1) streams with regime-switching drift across four reference ETFs (SPY, PFF, VNQ, BWX).
We evaluate 12 macro-transitions (widths $w \in \{100, 500, 1000\}$, logistic parameter interpolation, $n = 8000$, $\tau^* = 4000$) across three regimes: IID ($\alpha + \beta = 0$), Cal.~A ($\alpha + \beta \approx 0.95$), and Cal.~B ($\alpha + \beta \approx 0.99$) with Student-$t_7$ innovations.
Each setting is tested over \textbf{30 random seeds} (360 streams per table cell).
Features are log-return $r_t$ and 20-period rolling volatility.
Detection uses a tolerance window ($\tau = \max(1000, w)$, first-match unilateral matching; PHT threshold $\lambda = 15$).

\noindent\textbf{Results.}
Table~\ref{tab:results_concept} reports 14 pipeline configurations, including the static-RF rows of Section~\ref{sec:solution_rf}. All significance is assessed at the seed level (the unit of statistical independence) to avoid pseudo-replication~\cite{demsar_statistical_2006}; sign tests confirm complete separation on all starved pipelines ($30/30$ seeds, $p \approx 2 \times 10^{-9}$), with $0/1080$ detections for PHT+ARF (vs.\ $932/1080$ for PHT+HT) and EDDM+ARF (vs.\ $882/1080$ for EDDM+HT).

\begin{table*}[t]
  \centering
  \caption{Concept drift detection on heteroscedastic ProteuS streams (12 transitions $\times$ 30 seeds; 360 runs per cell). Mean $\pm$ 95\% bootstrap confidence intervals (1,000 resamples); $\infty$: no detection within the tolerance horizon; $[X]$: delay-corrected ADD (KSWIN intrinsic window delay $W/2 = 15$). Boldface marks the blind-spot collapse ($\mathrm{F1}=0.00$) and full resolution ($\mathrm{F1}=1.00$). Complete separations at the seed level hold for every starved pipeline against its baseline or windowed resolution (paired sign tests, 30/30 seeds favouring the non-starved arm, $p \approx 2\times10^{-9}$; see text for event counts).}
  \label{tab:results_concept}
  \begin{tabular}{@{}lrrrrrr@{}}
    \toprule
                                                   & \multicolumn{2}{c}{\textbf{IID Baseline}} & \multicolumn{2}{c}{\textbf{Cal. A (Low GARCH)}} & \multicolumn{2}{c}{\textbf{Cal. B (High GARCH)}}                                                                     \\
    \cmidrule(lr){2-3} \cmidrule(lr){4-5} \cmidrule(lr){6-7}
    \textbf{Pipeline Configuration}                & \textbf{F1}                               & \textbf{ADD}                                    & \textbf{F1}                                      & \textbf{ADD}       & \textbf{F1}             & \textbf{ADD}       \\
    \midrule
    \textbf{PHT + HT}                              & 0.75 {\small $\pm$0.04}                   & 23 {\small $\pm$3}                              & 0.86 {\small $\pm$0.04}                          & 22 {\small $\pm$2} & 0.98 {\small $\pm$0.02} & 23 {\small $\pm$3} \\
    \textbf{PHT + ARF} ($c_{\mathrm{int}} = 32$)   & 0.68 {\small $\pm$0.05}                   & 19 {\small $\pm$1}                              & 0.63 {\small $\pm$0.05}                          & 20 {\small $\pm$1} & 0.56 {\small $\pm$0.05} & 19 {\small $\pm$1} \\
    \textbf{PHT + ARF} ($c_{\mathrm{int}} = 1$)    & \textbf{0.00}                             & $\infty$                                        & \textbf{0.00}                                    & $\infty$           & \textbf{0.00}           & $\infty$           \\
    \midrule
    \textbf{EDDM + HT}                             & 0.68 {\small $\pm$0.05}                   & 37 {\small $\pm$4}                              & 0.83 {\small $\pm$0.04}                          & 39 {\small $\pm$3} & 0.94 {\small $\pm$0.03} & 40 {\small $\pm$4} \\
    \textbf{EDDM + ARF} ($c_{\mathrm{int}} = 1$)   & \textbf{0.00}                             & $\infty$                                        & \textbf{0.00}                                    & $\infty$           & \textbf{0.00}           & $\infty$           \\
    \midrule
    \textbf{ADWIN + HT}                            & 0.80 {\small $\pm$0.04}                   & 18 {\small $\pm$4}                              & 0.89 {\small $\pm$0.03}                          & 14 {\small $\pm$2} & 0.99 {\small $\pm$0.01} & 13 {\small $\pm$2} \\
    \textbf{ADWIN + ARF} ($c_{\mathrm{int}} = 32$) & 1.00                                      & 31                                              & 1.00                                             & 31                 & 1.00                    & 31                 \\
    \textbf{ADWIN + ARF} ($c_{\mathrm{int}} = 1$)  & 0.96 {\small $\pm$0.02}                   & 9                                               & 0.99 {\small $\pm$0.01}                          & 9                  & 1.00                    & 9                  \\
    \midrule
    \textbf{SRP + PHT} ($c_{\mathrm{int}} = 1$)    & \textbf{0.00}                             & $\infty$                                        & \textbf{0.00}                                    & $\infty$           & \textbf{0.00}           & $\infty$           \\
    \textbf{SRP + PHT} ($c_{\mathrm{int}} = 32$)   & 0.10 {\small $\pm$0.03}                   & 17 {\small $\pm$1}                              & 0.09 {\small $\pm$0.03}                          & 18 {\small $\pm$1} & 0.07 {\small $\pm$0.03} & 21 {\small $\pm$2} \\
    \textbf{KSWIN + ARF} ($c_{\mathrm{int}} = 1$)  & \textbf{1.00}                             & 14 [0]                                          & \textbf{1.00}                                    & 14 [0]             & \textbf{1.00}           & 14 [0]             \\
    \textbf{KSWIN + ARF} ($c_{\mathrm{int}} = 32$) & \textbf{1.00}                             & 14 [0]                                          & \textbf{1.00}                                    & 14 [0]             & \textbf{1.00}           & 14 [0]             \\
    \midrule
    \textbf{PHT + RF} (Static)                     & 0.73 {\small $\pm$0.05}                   & 19 {\small $\pm$1}                              & 0.84 {\small $\pm$0.04}                          & 20 {\small $\pm$2} & 0.96 {\small $\pm$0.02} & 20 {\small $\pm$2} \\
    \textbf{ADWIN + RF} (Static)                   & 0.79 {\small $\pm$0.04}                   & 14 {\small $\pm$2}                              & 0.89 {\small $\pm$0.03}                          & 13 {\small $\pm$2} & 0.99 {\small $\pm$0.01} & 12 {\small $\pm$1} \\
    \bottomrule
  \end{tabular}
\end{table*}

The results in Table~\ref{tab:results_concept} instantiate the theoretical mechanisms of Section~\ref{sec:blindspot}. Monitoring ARF under continuous checks ($c_{\mathrm{int}}=1$) collapses detection across all regimes ($\mathrm{ADD} = \infty$, $\mathrm{F1} = 0.00$), validating Proposition~\ref{prop:starvation} under ensemble acceleration (Section~\ref{sec:hydra}). Setting $c_{\mathrm{int}}=32$ delays tree replacement, letting PHT accumulate evidence and recovering F1 to $0.56$--$0.68$ ($\mathrm{ADD} \approx 19$). This corroborates the Decoupling Principle (Definition~\ref{def:decoupling}).

Starvation requires both an adaptive ensemble and a cumulative monitor. Monitoring ARF with ADWIN preserves high performance ($\mathrm{F1} \ge 0.96$ at $c=1$) because window-divergence monitors track sub-window means rather than accumulating residuals (Remark~\ref{rem:agnostic}). Conversely, pairing PHT with a non-adaptive Hoeffding Tree eliminates starvation ($\mathrm{F1} \ge 0.75$). EDDM exhibits the same collapse under $c_{\mathrm{int}}=1$ ($\mathrm{F1} = 0.00$, $0/1080$ detections vs $882/1080$ for EDDM+HT), confirming that Proposition~\ref{prop:starvation} extends beyond CUSUM.\footnote{We report EDDM rather than DDM because default thresholds for DDM generate high false-alarm rates across both ARF and HT on these streams, preventing an isolated recall comparison.} Evaluating Streaming Random Patches (SRP)~\cite{gomes_srp_2019} with PHT shows the same starvation ($0/1080$ detections vs $932/1080$ for HT): the blind spot is not specific to ARF. The static RF (Section~\ref{sec:solution_rf}) resolves it ($\mathrm{F1} \ge 0.73$).

\textbf{KSWIN+ARF on ProteuS, and its limits.}
KSWIN~\cite{raab_kswin_2020} detects all ProteuS drifts ($\mathrm{F1}=1.00$, $\mathrm{ADD}=14$). Yet two caveats qualify this result. First, false-alarm budgets differ: matching pre-change alarm rates in Table~\ref{tab:results_concept} requires configuring ADWIN $45\times$ more conservatively and KSWIN $4.5\times$ more permissively than CUSUM. Second, ProteuS has zero pre-change error by construction, keeping KSWIN's sliding window invariant before drift. On the Bernoulli streams ($p_0 > 0$), this configuration triggers 8--9 false alarms per thousand pre-change steps across all $\alpha$, and its detection rate ($0.713$) falls below CUSUM ($0.928$) and ADWIN ($0.885$). This detector-side remedy thus fails once $p_0 > 0$.

\subsection{Bridging the Gap: The Regime Crossover}\label{sec:crossover}

The abrupt transitions of SEA~\cite{street_kim_2001} induce limited error shocks: with two informative features uniform on $[0, 10]$, one noise feature, and label thresholds $\theta \in \{8, 9, 7, 9.5\}$, a shift flips a fraction $|\theta_b^2 - \theta_a^2|/200$ of labels, yielding $\Delta e \in \{0.085, 0.16, 0.21\}$. While SEA operates on bounded uniform noise rather than unbounded Gaussian features (Section~\ref{sec:race}), cumulative error monitors observe only the resulting binary error sequence $e_t$. Because SEA's error shocks remain bounded at $\Delta e \le 0.21$, they stay below the Blind Spot zone ($\Delta e \ge 0.25$). To explore higher magnitudes, our crossover experiment sweeps the boundary shift of Section~\ref{sec:race} across $\Delta e \in [0.02, 0.50]$ over 100 independent streams ($8{,}000$ steps, abrupt drift at $t^* = 4{,}000$). An external River PHT ($\lambda=25$, $\delta_P=0.005$) monitors the binary error stream of both a baseline Hoeffding Tree and an ARF ($c_{\mathrm{int}}=1$). Both pipelines run from stream inception, re-arm upon alarm, score pre-drift alerts as false positives, and enforce a $1{,}000$-step post-drift detection horizon.

\begin{figure*}[t]
  \centering
  \includegraphics[width=0.90\textwidth]{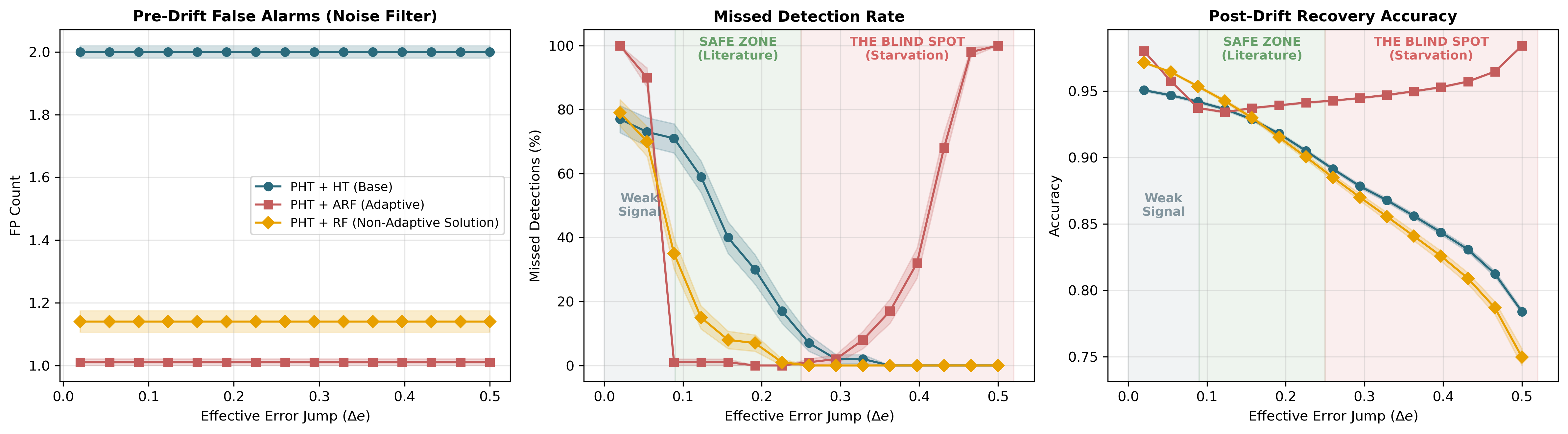}
  \caption{Regime crossover and non-adaptive resolution (100 streams/point, bands = $\pm 1$ SEM). Left: pre-drift false alarms; center: missed detections; right: post-drift accuracy. Shaded zones: Weak Signal, Safe Zone (ARF outperforms HT, the literature consensus) and Blind Spot ($\Delta e{\ge}0.25$), where the ARF's recovery starves the PHT ($100\%$ missed at $\Delta e{=}0.50$). The non-adaptive RF (yellow) inherits the ARF's variance reduction (left) and keeps the HT's persistent error signal (right), which removes starvation (center).}
  \label{fig:crossover_and_solution}
\end{figure*}

\textbf{Joint interpretation of the three crossover regimes.}
Figure~\ref{fig:crossover_and_solution} illustrates the operational trade-off across three regimes:
\begin{enumerate}[nosep, leftmargin=*, label=\textbf{(R\arabic*)}]
  \item \textbf{Weak Signal} ($\Delta e < 0.09$). Cumulative evidence fails to cross threshold within tolerance. Single HT models exhibit an $80\%$ miss rate due to sample variance: unbagged tree updates trigger localized error bursts that prematurely deplete the monitor's reference, whereas ARF bagging suppresses pre-drift false alarms by half (left panel, mirroring Section~\ref{sec:hardware}).
  \item \textbf{Safe Zone} ($\Delta e \in [0.09, 0.25]$). Typical of historical benchmarks like SEA. The ARF filters baseline noise (left panel) while allowing sufficient cumulative evidence to reach threshold, achieving $0\%$ missed detections and outperforming HT (center panel).
  \item \textbf{Blind Spot} ($\Delta e \ge 0.25$). Accelerated adaptation (Section~\ref{sec:hydra}) shortens transient length $W$. Driven by incremental learning in surviving trees ($98.6\%$ of erasure; Section~\ref{sec:race}), post-drift accuracy recovers immediately (right panel), starving the external PHT until missed detections reach $100\%$ at $\Delta e = 0.50$ (center panel). The static RF resolves this starvation ($0\%$ misses, center) by preserving persistent error, at the cost of a ${\sim}24$-point drop in recovery accuracy (right panel).
\end{enumerate}

\textbf{Mechanism of the empirical crossover.}
Classical first-passage bounds for transient CUSUM~\cite{tartakovsky_2023_transient,tartakovsky_2014} (Markov, Doob reflection, Siegmund ARL) predict a roughly \emph{uniform} missed-detection rate across $\Delta e \in [0.05, 0.50]$ under canonical assumptions ($\lambda=25$, $\delta_P=0.005$, $\tau_{\mathrm{ARF}} \approx 18.5(\Delta e)^{-0.98}$), failing to reproduce the sharp transition in Figure~\ref{fig:crossover_and_solution}. This discrepancy stems from River's running mean-tracking baseline~\cite{montiel_river_2021} (which differs from the fixed-$p_0$ CUSUM of Eq.~\eqref{eq:cusum}) interacting with low-magnitude noise swaps (Remark~\ref{rem:bgswap}). Instead, the crossover qualitatively mirrors the non-monotonic profile of the peak accumulated evidence $\mathbb{E}[S_{\max}(\Delta e)]$ introduced in Definition~\ref{def:decoupling}. This expected maximum evidence is unimodal in $\Delta e$: it rises from $5.28$ $[5.01, 5.55]$ at $\Delta e = 0.028$ to a plateau near $33$ across $\Delta e \in [0.19, 0.25]$, then declines to $19.13$ $[18.73, 19.50]$ at $\Delta e = 0.498$. Past this plateau, stronger drifts generate \emph{less} cumulative evidence because adaptation extinguishes the transient faster than evidence accumulates. The three panels of Figure~\ref{fig:pht_scenarios} materialize this single evidence curve under three horizontal threshold cuts: $\lambda = 50$ exceeds the peak everywhere (total starvation), $\lambda = 25$ intersects it (producing the U-shaped blind spot), and $\lambda = 8$ runs below it except at the weakest drifts. Consequently, missed detections escalate with drift magnitude only beyond the peak of $\mathbb{E}[S_{\max}(\Delta e)]$ ($\Delta e > 0.25$), where accelerated adaptation erases transient error faster than CUSUM accumulation.

\textbf{Real-world streams.}
We evaluate six real-world streams: the Bank Account Fraud (BAF) tabular benchmark~\cite{jesus_baf_2022} in its three temporal-split variants (Base, Variant I, Variant II), and three variants of the INSECTS benchmark~\cite{souza_insects_2020} (\emph{abrupt\_balanced}, \emph{gradual\_balanced}, \emph{incremental\_reoccurring\_balanced}). We compute F1 over 30 random seeds per variant, evaluating true detections via a bipartite Hopcroft-Karp matching protocol~\cite{hopcroft_karp_1973}. The effective error jump $\Delta e$ is estimated drift-by-drift via an adaptive window protocol $w_k = \min(5000, \min(\Delta t_{\mathrm{left}}, \Delta t_{\mathrm{right}})/3)$ to avoid the overlapping-window artefact that fixed-window estimators produce on closely-spaced engineered drifts. Table~\ref{tab:real_data_summary} places the six streams on the crossover; all INSECTS positions follow Souza et al.~\cite{souza_insects_2020}, Table~2.

The three BAF variants lie in the Weak Signal zone ($\Delta e \approx 0$), where the blind spot cannot manifest. INSECTS \emph{abrupt\_balanced} ($\Delta e \approx 0.07$) is uninformative: its five canonical transitions exhibit heterogeneous, partially negative shifts ($\Delta e \in \{0.68, -0.15, 0.24, -0.27, -0.13\}$). Because error decreases on negative shifts, neither pipeline reliably tracks the transitions; both alarm at rates matching their stationary background noise, meaning the marginal F1 difference ($0.18$ for ARF vs.\ $0.13$ for HT; Table~\ref{tab:real_data_summary}) reflects false-alarm timing variance rather than genuine detection.

The failure mode appears clearly on the two high-magnitude variants through frequent false alarms (Remark~\ref{rem:flooding}): \emph{gradual\_balanced} ($\Delta e \approx 0.45$) and \emph{incremental\_reoccurring\_balanced} ($\Delta e \approx 0.39$). On \emph{gradual\_balanced}, PHT+HT detects the transition with $7$ alarms (precision $0.14$, $\mathrm{F1} = 0.25$). In comparison, PHT+ARF ($c_{\mathrm{int}}=1$) records $86$ alarms, reducing precision to $0.012$ and $\mathrm{F1}$ to $0.02$ (a $10.57\times$ reduction in F1). The \emph{incremental\_reoccurring\_balanced} variant shows a similar decline ($1.61\times$ ratio). Because the benchmark stream is fixed and single Hoeffding Trees are deterministic, the HT baseline yields identical outputs across runs, preventing cross-seed paired hypothesis tests. Evaluating these differences statistically will require stream-level resampling protocols, such as block bootstrapping.

These real-world results corroborate the operational duality formalized in Remark~\ref{rem:flooding}. While clean stationary streams produce a recall collapse via transient starvation (Proposition~\ref{prop:starvation}), streams with substantial baseline error ($p_0 \in [0.16, 0.26]$ on INSECTS vs.\ $p_0 = 0$ on ProteuS) trigger recurrent premature tree swaps (Remark~\ref{rem:bgswap}). Each substitution injects an untrained background tree, generating false alarms that collapse precision. Tuning the threshold cannot resolve this dilemma: any threshold low enough to capture real transients triggers massive flooding from background swaps, and even at the threshold allowing a single false alarm, precision peaks at $0.33$.\footnote{Two further INSECTS variants are excluded: \emph{incremental\_balanced} lacks discrete drift boundaries, preventing bipartite event matching, while \emph{incremental\_abrupt\_balanced} exhibits an overall decrease in classification difficulty ($\Delta e = -0.094$).}

\begin{table*}[ht]
  \centering
  \caption{Empirical evaluation across six real-world benchmark streams (30 seeds per configuration). Drift positions follow the canonical protocol of Souza et al.~\cite{souza_insects_2020} (\textit{incremental\_reoccurring\_balanced} retains $K=2$ in-stream transitions). The $c_{\mathrm{int}}=1$ evaluation mode induces precision collapse via false-alarm flooding rather than starvation (e.g., 86 alarms emitted on \textit{gradual\_balanced} for a single true drift). $^{\ast}$\textit{abrupt\_balanced} exhibits heterogeneous, partly negative local error shifts where neither detector exceeds baseline false-alarm rates.}
  \label{tab:real_data_summary}
  \small
  \setlength{\tabcolsep}{8pt}
  \begin{tabular}{@{}lccccc@{}}
    \toprule
    Benchmark Variant                         & $\Delta e$ (adaptive) & $\text{F1}_{\mathrm{PHT+HT}}$ & $\text{F1}_{\mathrm{PHT+ARF}_{c=1}}$ & $\text{F1}_{\mathrm{ADW+ARF}_{c=1}}$ & Ratio HT / ARF \\
    \midrule
    BAF Base                                  & $\approx 0$           & 0.10                          & 0.10                                 & 0.00                                 & 1.00           \\
    BAF Variant I                             & $\approx 0$           & 0.09                          & 0.11                                 & 0.00                                 & 0.83           \\
    BAF Variant II                            & $\approx 0$           & 0.00                          & 0.00                                 & 0.00                                 & --             \\
    INSECTS \textit{abrupt\_balanced}         & $0.07^{\ast}$         & 0.13                          & 0.18                                 & \textbf{0.48}                        & 0.71           \\
    INSECTS \textit{gradual\_balanced}        & \textbf{0.45}         & 0.25                          & 0.02                                 & 0.06                                 & \textbf{10.57} \\
    INSECTS \textit{incremental\_reoccurring} & 0.39                  & 0.31                          & 0.19                                 & 0.16                                 & \textbf{1.61}  \\
    \bottomrule
  \end{tabular}
\end{table*}

\subsection{A Constructive Solution: The Non-Adaptive Ensemble}\label{sec:solution_rf}

The ablations of Section~\ref{sec:race} show that blocking replacements alone lifts detection only to $49/100$: the surviving trees still absorb the transient. This motivates a \emph{Non-Adaptive Random Forest} (RF)---static Bagging of Hoeffding Trees without internal detectors---wherever cumulative monitors are vulnerable to starvation. Because drift becomes persistent ($W \to \infty$), net CUSUM increments remain positive ($\Delta e - \delta_P > 0$), guaranteeing threshold crossing almost surely.

As shown in Figure~\ref{fig:crossover_and_solution}, the static RF decouples noise filtering from drift visibility: bagging keeps most of the ARF's reduction in pre-drift false alarms (left panel), while persistent classification error ensures $0\%$ missed detections across the entire blind-spot zone ($\Delta e \ge 0.25$, center panel).

On heteroscedastic ProteuS streams (Section~\ref{sec:proteus}; Table~\ref{tab:results_concept}), substituting static RF for ARF fully restores cumulative monitoring: PHT+ARF detects $0/1080$ shifts ($\mathrm{F1}=0.00$), whereas PHT+RF recovers $\mathrm{F1}$ to $0.73$--$0.96$ ($913/1080$ detections). Evaluating paired seed aggregates ($36$ streams/seed) confirms complete statistical separation across all $30$ independent seeds ($30/30$ seed-level wins favouring RF, two-sided sign test $p = 2 \times 0.5^{30} \approx 2 \times 10^{-9}$).

This architectural fix is strictly specific to the starvation regime: when paired with a window-divergence monitor such as ADWIN (which does not starve at $c_{\mathrm{ext}}=1$), the adaptive ARF detects more drifts than the static RF ($1063/1080$ vs.\ $959/1080$, $30/30$ seeds favouring ARF, $p \approx 2 \times 10^{-9}$) while maintaining continuous online learning. Disabling adaptation carries an inescapable predictive cost: under the severe synthetic shift ($\Delta e = 0.50$), static RF accuracy decays alongside HT, falling ${\sim}24$ points below the recovered ARF ($\approx 0.745$ vs.\ $\approx 0.985$; Figure~\ref{fig:crossover_and_solution}, right).

\section{Discussion}\label{sec:discussion}

\subsection{Practical Implications}
Default configurations (e.g., River's ARF at $c_{\mathrm{int}} = 1$) mask strong drifts from cumulative monitors.

\textbf{Where the admissible set is empty.} Calibrated via Definition~\ref{def:decoupling} on Gaussian streams (StrictCUSUM, $\delta_P = 0.01$, ARF $M=10$), reliable detection over $\Delta e \in [0.20, 0.40]$ requires capping $\lambda \le \lambda_{\mathrm{op}} \approx 21.93$, satisfying false-alarm budgets on that band. However, extending this target range breaks the calibration: below $\Delta e_c = 0.120$, transient evidence is too weak to cross any threshold meeting false-alarm constraints; at $\Delta e = 0.50$, accelerated adaptation pushes peak evidence below $\lambda_{\mathrm{FA}}$ in $62.6\%$ of bootstrap replicates. This infeasibility is not a defect of cumulative detectors per se, but an intrinsic bound on the range of magnitudes a single static threshold can cover.
Three architectural options follow:
\begin{enumerate}[nosep, leftmargin=*, label=\textbf{(S\arabic*)}]
  \item \textbf{Detector-side.} Window-divergence tests (e.g., KSWIN) compare sub-window distributions rather than accumulating residuals, restoring detection on ProteuS. However, this benefit vanishes on noisy streams with non-zero baseline error ($p_0 > 0$), where window tests trigger recurrent pre-drift false alarms (Section~\ref{sec:proteus}).
  \item \textbf{Classifier-side.} A static Random Forest without internal detectors ($\tau_{\mathrm{ARF}} = \infty$) eliminates the race condition, at a post-drift accuracy cost of up to ${\sim}24$ points on complete shifts (Section~\ref{sec:solution_rf}).
  \item \textbf{Two-tier architecture (untested).} The external monitor commands tree replacements. However, commanding swaps alone only achieves $49\%$ detection ($\lambda = 50$, $\Delta e = 0.33$), because incremental learning in surviving trees drives $98.6\%$ of signal erasure (Section~\ref{sec:race}); the pipeline must also throttle or buffer incremental leaf updates until an alarm fires.
\end{enumerate}

\subsection{Limitations and Future Work}\label{sec:limitations}
\textbf{Limitations.} \emph{First}, tracking internal swaps aggregates true detections with background noise swaps, preventing isolated $\mathrm{ARL}_0$ measurement. \emph{Second}, unlike starvation (formalized via stopping-time bounds in Proposition~\ref{prop:starvation}), false-alarm flooding on noisy streams lacks an analytical boundary. \emph{Third}, real-world streaming benchmarks with clean stationary baselines are scarce: tabular streams (BAF, INSECTS) feature non-zero baseline error (INSECTS: $p_0 \in [0.16, 0.26]$), shifting the operational failure mode from recall collapse (starvation) toward precision collapse (flooding). \emph{Fourth}, PRNG sharing across trees in River violates conditional independence, so only the distribution-free race bound applies (Proposition~\ref{prop:starvation_boundary}), and it saturates. Finally, the phenomenon was evaluated on abrupt transitions; gradual and multidimensional drift generators (e.g., Hyperplane) remain to be explored.

\textbf{Future Work.} \emph{First}, evaluating other streaming ensembles (e.g., Leveraging Bagging~\cite{bifet_leveraging_2010}) to map the universality of the race condition. \emph{Second}, formalizing the two-tier architecture (S3) with update gating to preserve transients during monitoring. \emph{Third}, establishing a stream-level replication protocol (e.g., temporal block bootstrap) to enable paired hypothesis tests across fixed real-world streams (Section~\ref{sec:crossover}). \emph{Fourth}, introducing an alarm cooldown (refractory period) to suppress cascading false alarms after background tree swaps on noisy streams. \emph{Fifth}, evaluating input-distribution monitors on $P(X)$ (e.g., multivariate two-sample tests), which bypass the classifier's error stream by construction, though they cannot detect posterior shifts $P(Y|X)$ lacking covariate motion. \emph{Sixth}, investigating streaming neural networks~\cite{gama_survey_2014}, where adaptation is gradient-based rather than tree-based, to measure whether continuous backpropagation erases error transients similarly.

\section{Conclusion}\label{sec:conclusion}

This paper analyzed the \textbf{blind spot paradox}, in which an adaptive classifier absorbs an error transient before an external cumulative monitor crosses its threshold. While preemption (the first swap preceding the alarm) marks the race online, signal erasure drives the detection failure: faster adaptation destroys the monitor's input. Ablations isolate the dominant factor: surviving trees absorb $98.6\%$ of the transient through leaf updates, whereas the first tree swap accounts for $0.71\%$ of erased volume while cutting detection by $31$ percentage points. Proposition~\ref{prop:starvation} formalizes this starvation for cumulative-evidence monitors. While windowed and two-sample tests escape this specific mechanism, they encounter separate trade-offs under baseline noise (Remark~\ref{rem:agnostic}). Removing adaptation via static ensembles bypasses the race entirely, but incurs an accuracy penalty of up to ${\sim}24$ points on complete shifts.

Detectability is governed by the transient evidence accumulated against the monitor's threshold. Because this evidence peaks at intermediate drift magnitudes, stronger drifts paradoxically become harder to detect. Below magnitude floor $\Delta e_c = 0.120$, no cumulative threshold is admissible; over $\Delta e \in [0.20, 0.40]$, calibration yields $\lambda_{\mathrm{op}} = 21.93$ $[19.88, 22.40]$. Threshold placement relative to this evidence governs the failure mode: missed drifts on stationary streams (recall collapse) versus false-alarm flooding on noisy baselines (precision collapse). A two-tier architecture, in which the external monitor gates classifier adaptation, is left to future work.

\IEEEtriggeratref{17}
\bibliographystyle{IEEEtran}
\bibliography{IEEEabrv,articleA_biblio_v84}

\begin{thebibliography}{10}
\providecommand{\url}[1]{#1}
\csname url@samestyle\endcsname
\providecommand{\newblock}{\relax}
\providecommand{\bibinfo}[2]{#2}
\providecommand{\BIBentrySTDinterwordspacing}{\spaceskip=0pt\relax}
\providecommand{\BIBentryALTinterwordstretchfactor}{4}
\providecommand{\BIBentryALTinterwordspacing}{\spaceskip=\fontdimen2\font plus
\BIBentryALTinterwordstretchfactor\fontdimen3\font minus
  \fontdimen4\font\relax}
\providecommand{\BIBforeignlanguage}[2]{{%
\expandafter\ifx\csname l@#1\endcsname\relax
\typeout{** WARNING: IEEEtran.bst: No hyphenation pattern has been}%
\typeout{** loaded for the language `#1'. Using the pattern for}%
\typeout{** the default language instead.}%
\else
\language=\csname l@#1\endcsname
\fi
#2}}
\providecommand{\BIBdecl}{\relax}
\BIBdecl

\bibitem{gomes_arf_2017}
H.~M. Gomes, A.~Bifet, J.~Read, J.~P. Barddal, F.~Enembreck, B.~Pfahringer,
  G.~Holmes, and T.~Abdessalem, ``Adaptive random forests for evolving data
  stream classification,'' \emph{Machine Learning}, vol. 106, no. 9--10, pp.
  1469--1495, 2017.

\bibitem{bifet_adwin_2007}
A.~Bifet and R.~Gavald{\`a}, ``Learning from time-changing data with adaptive
  windowing,'' in \emph{Proc. 7th {SIAM} Int. Conf. Data Mining ({SDM})}, 2007,
  pp. 443--448.

\bibitem{gama_survey_2014}
J.~Gama, I.~{\v{Z}}liobait{\.{e}}, A.~Bifet, M.~Pechenizkiy, and A.~Bouchachia,
  ``A survey on concept drift adaptation,'' \emph{ACM Computing Surveys},
  vol.~46, no.~4, pp. 44:1--44:37, 2014.

\bibitem{lu_concept_drift_2018}
J.~Lu, A.~Liu, F.~Dong, F.~Gu, J.~Gama, and G.~Zhang, ``Learning under concept
  drift: {A} review,'' \emph{{IEEE} Trans. Knowl. Data Eng.}, vol.~31, no.~12,
  pp. 2346--2363, 2019.

\bibitem{klinkenberg_2000}
R.~Klinkenberg and T.~Joachims, ``Detecting concept drift with {Support Vector
  Machines},'' in \emph{Proc. 17th Int. Conf. Machine Learning ({ICML})}.\hskip
  1em plus 0.5em minus 0.4em\relax San Francisco, CA: Morgan Kaufmann, 2000,
  pp. 487--494.

\bibitem{bayram_2022}
F.~Bayram, B.~S. Ahmed, and A.~Kassler, ``From concept drift to model
  degradation: {A}n overview on performance-aware drift detectors,''
  \emph{Knowledge-Based Systems}, vol. 245, p. 108632, 2022.

\bibitem{montiel_river_2021}
\BIBentryALTinterwordspacing
J.~Montiel, M.~Halford, S.~M. Mastelini, G.~Bolmier, R.~Sourty, R.~Vaysse,
  A.~Zouitine, H.~M. Gomes, J.~Read, T.~Abdessalem, and A.~Bifet, ``{River}:
  Machine learning for streaming data in {Python},'' \emph{Journal of Machine
  Learning Research}, vol.~22, no. 110, pp. 1--8, 2021. [Online]. Available:
  \url{https://jmlr.org/papers/v22/20-1380.html}
\BIBentrySTDinterwordspacing

\bibitem{suarez_cetrulo_proteus_2025}
A.~L. Su{\'a}rez-Cetrulo, A.~Cervantes, and D.~Quintana, ``{ProteuS}: {A}
  generative approach for simulating concept drift in financial markets,''
  \emph{arXiv preprint arXiv:2509.11844}, 2025.

\bibitem{page_1954}
E.~S. Page, ``Continuous inspection schemes,'' \emph{Biometrika}, vol.~41, no.
  1--2, pp. 100--115, 1954.

\bibitem{basseville_1993}
M.~Basseville and I.~V. Nikiforov, \emph{Detection of Abrupt Changes: {Theory}
  and {Application}}.\hskip 1em plus 0.5em minus 0.4em\relax Englewood Cliffs,
  NJ: Prentice-Hall, 1993.

\bibitem{tartakovsky_2014}
A.~G. Tartakovsky, I.~V. Nikiforov, and M.~Basseville, \emph{Sequential
  Analysis: {Hypothesis} {Testing} and {Changepoint} {Detection}}.\hskip 1em
  plus 0.5em minus 0.4em\relax Boca Raton, FL: Chapman and Hall/CRC, 2014.

\bibitem{moustakides_1986}
G.~V. Moustakides, ``Optimal stopping times for detecting changes in
  distributions,'' \emph{The Annals of Statistics}, vol.~14, no.~4, pp.
  1379--1387, 1986.

\bibitem{lorden_1971}
G.~Lorden, ``Procedures for reacting to a change in distribution,'' \emph{The
  Annals of Mathematical Statistics}, vol.~42, no.~6, pp. 1897--1908, 1971.

\bibitem{goncalves_2014}
P.~M. Gon{\c{c}}alves, Jr., S.~G.~T. de~Carvalho~Santos, R.~S.~M. Barros, and
  D.~C.~L. Vieira, ``A comparative study on concept drift detectors,''
  \emph{Expert Systems with Applications}, vol.~41, no.~18, pp. 8144--8156,
  2014.

\bibitem{lukats_2025}
D.~Lukats, O.~Zielinski, A.~Hahn, and F.~Stahl, ``A benchmark and survey of
  fully unsupervised concept drift detectors on real-world data streams,''
  \emph{International Journal of Data Science and Analytics}, vol.~19, no.~1,
  pp. 1--31, 2025.

\bibitem{stirling_2018}
M.~Stirling, Y.~S. Koh, P.~Fournier-Viger, and S.~D. Ravana, ``Concept drift
  detector selection for {Hoeffding} adaptive trees,'' in \emph{{AI} 2018:
  Advances in Artificial Intelligence}, ser. {LNCS}, vol. 11320.\hskip 1em plus
  0.5em minus 0.4em\relax Springer, 2018, pp. 730--736.

\bibitem{read_2018}
J.~Read, ``Concept-drifting data streams are time series; {The} case for
  continuous adaptation,'' \emph{arXiv preprint arXiv:1810.02266}, 2018.

\bibitem{gama_ddm_2004}
J.~Gama, P.~Medas, G.~Castillo, and P.~Rodrigues, ``Learning with drift
  detection,'' in \emph{Advances in Artificial Intelligence ({SBIA} 2004)},
  ser. {LNCS}, vol. 3171.\hskip 1em plus 0.5em minus 0.4em\relax Springer,
  2004, pp. 286--295.

\bibitem{baena_eddm_2006}
M.~Baena-Garc{\'\i}a, J.~del Campo-{\'A}vila, R.~Fidalgo-Merino, A.~Bifet,
  R.~Gavald{\`a}, and R.~Morales-Bueno, ``{Early Drift Detection Method},'' in
  \emph{Proc. 4th {ECML} {PKDD} Int. Workshop on Knowledge Discovery from Data
  Streams}, vol.~6, Berlin, Germany, 2006, pp. 77--86.

\bibitem{ross_ecdd_2012}
G.~J. Ross, N.~M. Adams, D.~K. Tasoulis, and D.~J. Hand, ``Exponentially
  weighted moving average charts for detecting concept drift,'' \emph{Pattern
  Recognition Letters}, vol.~33, no.~2, pp. 191--198, 2012.

\bibitem{raab_kswin_2020}
C.~Raab, M.~Heusinger, and F.-M. Schleif, ``Reactive soft prototype computing
  for concept drift streams,'' \emph{Neurocomputing}, vol. 416, pp. 340--351,
  2020.

\bibitem{david_order_2003}
H.~A. David and H.~N. Nagaraja, \emph{Order Statistics}, 3rd~ed.\hskip 1em plus
  0.5em minus 0.4em\relax Hoboken, NJ: Wiley, 2003.

\bibitem{bifet_hat_2009}
A.~Bifet and R.~Gavald{\`a}, ``Adaptive learning from evolving data streams,''
  in \emph{Advances in Intelligent Data Analysis VIII ({IDA} 2009)}, ser.
  {LNCS}, vol. 5772.\hskip 1em plus 0.5em minus 0.4em\relax Springer, 2009, pp.
  249--260.

\bibitem{tartakovsky_2005_multichart}
A.~G. Tartakovsky, ``Asymptotic performance of a multichart {CUSUM} test under
  false alarm probability constraint,'' in \emph{Proc. 44th {IEEE} Conf.
  Decision and Control and European Control Conf. ({CDC}-{ECC}'05)}, 2005, pp.
  320--325.

\bibitem{nikiforov_2012}
B.~K. Gu{\'e}pi{\'e}, L.~Fillatre, and I.~V. Nikiforov, ``Sequential detection
  of transient changes,'' \emph{Sequential Analysis}, vol.~31, no.~4, pp.
  528--547, 2012.

\bibitem{tartakovsky_2023_transient}
G.~Sokolov, V.~S. Spivak, and A.~G. Tartakovsky, ``Detecting an intermittent
  change of unknown duration,'' \emph{Sequential Analysis}, vol.~42, no.~3, pp.
  269--302, 2023.

\bibitem{hoeffding_1963}
W.~Hoeffding, ``Probability inequalities for sums of bounded random
  variables,'' \emph{Journal of the American Statistical Association}, vol.~58,
  no. 301, pp. 13--30, 1963.

\bibitem{demsar_statistical_2006}
\BIBentryALTinterwordspacing
J.~Dem{\v{s}}ar, ``Statistical comparisons of classifiers over multiple data
  sets,'' \emph{Journal of Machine Learning Research}, vol.~7, no.~1, pp.
  1--30, 2006. [Online]. Available:
  \url{https://jmlr.org/papers/v7/demsar06a.html}
\BIBentrySTDinterwordspacing

\bibitem{gomes_srp_2019}
H.~M. Gomes, J.~Read, and A.~Bifet, ``Streaming random patches for evolving
  data stream classification,'' in \emph{Proc. 2019 {IEEE} Int. Conf. Data
  Mining ({ICDM})}, 2019, pp. 240--249.

\bibitem{street_kim_2001}
W.~N. Street and Y.~Kim, ``A streaming ensemble algorithm ({SEA}) for
  large-scale classification,'' in \emph{Proc. 7th {ACM} {SIGKDD} Int. Conf.
  Knowledge Discovery and Data Mining ({KDD})}, 2001, pp. 377--382.

\bibitem{jesus_baf_2022}
S.~Jesus, J.~Pombal, D.~Alves, A.~Cruz, P.~Saleiro, R.~P. Ribeiro, J.~Gama, and
  P.~Bizarro, ``Turning the tables: {B}iased, imbalanced, dynamic tabular
  datasets for {ML} evaluation,'' in \emph{Advances in Neural Information
  Processing Systems ({NeurIPS} 2022, Datasets and Benchmarks)}, vol.~35.\hskip
  1em plus 0.5em minus 0.4em\relax Curran Associates, Inc., 2022, pp.
  33\,563--33\,575.

\bibitem{souza_insects_2020}
V.~M.~A. Souza, D.~M. dos Reis, A.~G. Maletzke, and G.~E. A. P.~A. Batista,
  ``Challenges in benchmarking stream learning algorithms with real-world
  data,'' \emph{Data Mining and Knowledge Discovery}, vol.~34, no.~6, pp.
  1805--1858, 2020.

\bibitem{hopcroft_karp_1973}
J.~E. Hopcroft and R.~M. Karp, ``An {$n^{5/2}$} algorithm for maximum matchings
  in bipartite graphs,'' \emph{SIAM Journal on Computing}, vol.~2, no.~4, pp.
  225--231, 1973.

\bibitem{bifet_leveraging_2010}
A.~Bifet, G.~Holmes, and B.~Pfahringer, ``Leveraging {Bagging} for evolving
  data streams,'' in \emph{Machine Learning and Knowledge Discovery in
  Databases ({ECML} {PKDD} 2010)}, ser. {LNCS}, vol. 6321.\hskip 1em plus 0.5em
  minus 0.4em\relax Springer, 2010, pp. 135--150.

\end{thebibliography}

\end{document}